\documentclass{article}

     \usepackage[final,nonatbib]{optrl_2019}

\usepackage[utf8]{inputenc} 
\usepackage[T1]{fontenc}    
\usepackage{hyperref}       
\usepackage{url}            
\usepackage{booktabs}       
\usepackage{amsfonts}       
\usepackage{nicefrac}       
\usepackage{microtype}      

\usepackage{amsthm}
\usepackage{amsmath}
\usepackage{amsfonts} 
\usepackage{graphicx}

\title{Compatible features for Monotonic Policy Improvement}

\author{%
    Marcin B. Tomczak\thanks{Correspondence to mbt27@cam.ac.uk} \\
    PROWLER.io \\
    72 Hills Road, Cambridge, CB2 1LA, UK \\
    \texttt{marcin@prowler.io} \\
   \And
    Sergio Valcarcel Macua \\
    PROWLER.io \\
    72 Hills Road, Cambridge, CB2 1LA, UK \\
    \texttt{sergio@prowler.io} \\
   \And
    Enrique Munoz de Cote \\
    PROWLER.io  \\
    72 Hills Road, Cambridge, CB2 1LA, UK \\
    \texttt{enrique@prowler.io} \\
   \And
    Peter Vrancx \\
    PROWLER.io  \\
    72 Hills Road, Cambridge, CB2 1LA, UK \\
    \texttt{peter@prowler.io} \\
}

\newtheorem{theorem}{Theorem}

\newcommand{\refthm}[1]{Theorem \ref{#1}}

\newcommand{\refeqn}[1]{Equation (\ref{#1})}

\begin{document}
\maketitle

\begin{abstract}
Recent policy optimization approaches have achieved substantial empirical success by constructing surrogate optimization objectives. The Approximate Policy Iteration objective \cite{pmlr-v37-schulman15, Kakade:2002:AOA:645531.656005} has become a standard optimization target for reinforcement learning problems. Using this objective in practice requires an estimator of the advantage function. Policy optimization methods such as those proposed in \cite{schulman2015high} estimate the advantages using a parametric critic. In this work we establish conditions under which the parametric approximation of the critic does not introduce bias to the updates of surrogate objective. These results hold for a general class of parametric policies, including deep neural networks. We obtain a result analogous to the compatible features derived for the original Policy Gradient Theorem \cite{Sutton:1999:PGM:3009657.3009806}. As a result, we also identify a previously unknown bias that current state-of-the-art policy optimization algorithms \cite{pmlr-v37-schulman15, DBLP:journals/corr/SchulmanWDRK17} have introduced by not employing these compatible features. 
\end{abstract}

\section{Introduction}
Model-free reinforcement learning has been applied successfully to a variety of problems \cite{Duan:2016:BDR:3045390.3045531}. Nevertheless, many of the state-of-the-art approaches introduce bias to the policy updates and hence do not guarantee convergence. This might lead to oversensitivity of the algorithms to changes in initial conditions and has an adverserial influence on reproducibility \cite{DBLP:journals/corr/abs-1709-06560}. 

Policy updates often use values learned by a critic to  the variance of stochastic gradients. However, arbitrary parametrization of the critic can introduce bias into policy updates. Early work in \cite{Sutton:1999:PGM:3009657.3009806} provides restrictions on the form of the function approximator under which the policy gradient remains unbiased and the convergence of learning is guaranteed. The notion of compatible features introduced in \cite{Sutton:1999:PGM:3009657.3009806} also has connections to the natural policy gradient \cite{Kakade:2001:NPG:2980539.2980738}. Compatible features remain an active area of research \cite{DBLP:journals/corr/abs-1902-02823}.

Modern policy optimization algorithms focus on optimizing a surrogate objective that approximates the value of the policy \cite{pmlr-v37-schulman15,DBLP:journals/corr/SchulmanWDRK17}. This surrogate objective remains an accurate optimization target when the policy used to gather the data is close to the policy being the optimized \cite{pmlr-v37-schulman15, Kakade:2002:AOA:645531.656005}. In this work we investigate the conditions on the parametrization of critic under which the updates of the surrogate objective remain unbiased. 

\section{Preliminaries}
We assume a classical MDP formulation as in \cite{Sutton:1999:PGM:3009657.3009806}. We consider an MDP being a tuple $\langle S, A, P, R, \rho_0, \gamma \rangle$, $S$ is a set of states, $A$ is a set of actions, $P:S \times A  \times S \rightarrow \mathbb{R}_{+}$ is a transition model, $R:S \times A \rightarrow \mathbb{R}$ is a reward function, $\rho_0$ is an initial distribution over states $S$ and $\gamma \in (0,1)$ is a discount factor. We denote a trajectory as $\tau = (s_0,  a_0, s_1, a_1, \ldots )$. Given a stochastic policy $\pi: S \times A \rightarrow \mathbb{R}_{+}$, we use $\tau \sim \pi$ to denote that the trajectory has been generated by following policy $\pi$, i.e. $s_{t+1} \sim P(\cdot|s_t, a_t)$, $a_t \sim \pi(\cdot|s_t)$, and $s_0 \sim \rho_0(\cdot)$. 

We denote unnormalized discounted state occupancy measure as $\rho_\pi (s) = \mathbb{E}_{\tau\sim\pi} \sum_{t\ge 0} \gamma^t P(s_t=s)$.
The value function is defined as $V^\pi (s) = \mathbb{E}_{\tau\sim\pi}  [\sum_{t \ge 0} \gamma^t R(s_t, a_t) |s_0 = s]$ and the state-value function is $Q^\pi (s,a)= \mathbb{E}_{\tau\sim\pi}  [\sum_{t \ge 0} \gamma^t R(s_t, a_t) |s_0 = s, a_0 = a]$. 
The advantage function is defined as $A^\pi \left(s,a\right) = Q^\pi\left(s,a\right) - V^\pi\left(s\right)$. To emphasize that the policy $\pi$ is parametrized by parameters $\theta$ we use notation $\pi_\theta$. 
We consider policies such that $\frac{\partial \pi_\theta(a|s)}{\partial \theta}$ exist and is continuous for every $s \in S$ and $a \in A$. 
This is a general class of policies that includes, e.g., expressive deep neural networks.
Under this discounted cost setting, reinforcement learning algorithms seek a policy $\pi$ that maximizes its value defined as: 
\begin{equation}
J (\pi) = \int \rho_0(s) V^\pi(s) ds.
\end{equation}

When dealing with two different parametric policies, $\pi_\theta$ and $\pi_{\tilde\theta}$, with parameters $\theta$ and $\tilde\theta$, respectively,
we follow the notation in \cite{pmlr-v37-schulman15} and define a surrogate policy optimization objective as
\begin{equation}
    L_{\pi_{\theta}}(\pi_{\tilde\theta}) : = J(\pi_\theta) + \mathbb{E}_{s \sim \rho^\pi (\cdot), a \sim \pi_{\tilde\theta}(\cdot|s) } A^{\pi_{\theta}}(s,a).
\end{equation}
When the expected divergence between $\pi_\theta$ and $\pi_{\tilde\theta}$ is small $L_{\pi_{\theta}}(\pi_{\tilde\theta})$ can be treated as a good approximation of 
 $J(\pi_{\tilde\theta})$ \cite{pmlr-v37-schulman15, Kakade:2002:AOA:645531.656005}.

\section{Related work}
We begin by restating a well-known result presented in \cite{Sutton:1999:PGM:3009657.3009806} that allows to calculate the gradient of $J(\pi_\theta)$ w.r.t policy parameters $\theta$:
\begin{equation}
\label{eqn:pg_theorem}
\frac{\partial J(\pi_\theta)}{\partial\theta} = \int \rho_{\pi_\theta} (s) \int \pi_\theta (a|s) \frac{\partial\log \pi_\theta (a|s)}{\partial \theta} Q^{\pi_\theta} (s,a) dads.
\end{equation}
The remarkable property of the expression for the policy gradient given by \refeqn{eqn:pg_theorem} is that calculating the gradient $\frac{\partial J(\pi_\theta)}{\partial\theta}$ does not require calculating the gradient of the occupancy measure $\frac{\partial \rho_{\pi_\theta}}{\partial\theta} $. 
Thus, if we know $Q^{\pi_\theta}$, the gradient $\frac{\partial J(\pi_\theta)}{\partial\theta}$ can be approximated with Monte Carlo sampling using the trajectories obtained by following policy $\pi_\theta$. 

However, since $Q^{\pi_\theta} (s,a)$ is not known in advance, it has to be estimated. One possible choice of the estimator is to use empirical returns: $\hat Q^{\pi_\theta} (s,a) = \sum_{t \ge t'} \gamma^t r(s_t, a_t)$. The problem with this choice of estimator $\hat Q^{\pi_\theta} (s,a)$ is the large variance of the stochastic version of the gradients $\frac{\partial J(\pi_\theta)}{\partial\theta}$ which results in poor practical performance \cite{ciosek2018expected}. 

To reduce the variance in estimation of $\frac{\partial J(\pi_\theta)}{\partial\theta}$ many algorithms learn a parametric approximation $f_w (s,a)$ of $Q^{\pi_\theta}(s,a)$ by solving the regression problem: 
\begin{equation}
\label{eq:standard_regression}
     w^* = \text{argmin}_w \frac{1}{N} \sum_{i=1}^N \big(Q^{\pi_\theta}(s_i,a_i) - f_w(s_i,a_i)\big)^2
\end{equation}
where $\{(s_i, a_i) \}_{i=1}^N$ are state action pairs sampled with policy $\pi_\theta$. Using such $f_{w^*} (s,a)$ in place of $Q^{\pi_\theta}(s,a)$ can introduce bias to policy updates. The following Theorem derived in \cite{Sutton:1999:PGM:3009657.3009806, Konda:2003:AA:942271.942292} provides conditions under which function approximator $f_w(s,a)$ of $Q^{\pi_\theta} (s,a)$ does not introduce bias to $\frac{\partial J (\pi_\theta)}{\partial \theta}$.

\begin{theorem}
\label{thm:pg_func_approximator}
Let $f_w: S \times A \rightarrow \mathbb{R}$ be differentiable function approximator that satisfies the following conditions:
\begin{equation}
\label{eqn:condition_pg_unbiased}
\int \rho_{\pi_\theta}(s) \int  \pi_\theta (a|s) \big( Q^{\pi_\theta}(s,a) - f_w(s,a) \big) \frac{\partial f_w (s,a)}{\partial w} da ds =  0
\end{equation}
and
\begin{equation}
\label{eqn:condition_fw}
\frac{\partial f_w}{\partial w}(s,a) = \frac{1}{\pi_\theta (a|s)} \frac{\partial \pi_\theta (a|s)}{\partial \theta}.
\end{equation}
Then,
\begin{equation}
\frac{\partial J (\pi_\theta) }{\partial \theta}= \int \rho_{\pi_\theta}(s)\int \frac{\partial  \pi_\theta (a|s)}{\partial\theta} f_w (s,a) dads.
\end{equation}
\end{theorem}

Note that condition \eqref{eqn:condition_pg_unbiased} can be satisfied by solving the regression problem $w^* = \text{argmin}_w \int \rho_{\pi_\theta}(s) \int  \pi_\theta (a|s) \big( Q^{\pi_\theta}(s,a) - f_w(s,a) \big)^2 da ds $. Satisfying the condition given by \eqref{eqn:condition_fw} can be done by setting $f_w(s,a) = w^\top \frac{\partial\log \pi_\theta(a|s)}{\partial \theta} + c_0$. The constant $c_0$ can be set to $V^\pi(s)$ to ensure that $\mathbb{E}_{a \sim \pi_\theta(\cdot|s) } f_w(s,a) = w^\top \mathbb{E}_{a \sim \pi_\theta(\cdot|s) }\frac{\partial\log \pi_\theta(a|s)}{\partial \theta} + V^\pi(s)=  V^\pi(s)$, since the expectation $\mathbb{E}_{a \sim \pi_\theta(\cdot|s) }\frac{\partial\log \pi_\theta(a|s)}{\partial \theta}  = \int \frac{\partial \pi_\theta (a|s)}{\partial \theta} da =  \frac{\partial \int \pi_\theta(a|s) da}{\partial \theta}  = 0$ for every $s \in S$.  See the discussion in \cite{Sutton:1999:PGM:3009657.3009806} for details.

Performing gradient ascent on $J(\pi_\theta)$ requires resampling the data with every policy update as the expectation in \refeqn{eqn:pg_theorem} is taken over the distribution of trajectories induced by $\pi_\theta$. The approach presented in \cite{Kakade:2002:AOA:645531.656005, pmlr-v37-schulman15}  tackles the problem of improving policy $\pi_\theta$ in a different way. 

Given the data gathered using the current policy $\pi_\theta$ we want to estimate the lower bound on the performance of an arbitrary policy $\pi_{\tilde\theta}$ and perform maximisation w.r.t $\tilde\theta$. Since the lower bound is tight at $\pi_{\tilde\theta} = \pi_\theta$ maximizing this lower bound w.r.t $\tilde\theta$ guarantees an improvement in the value of policy \cite{pmlr-v37-schulman15}. Hence, sequential optimization of the discussed lower bound is called Monotonic Policy Improvement. The lower bound derived in \cite{pmlr-v37-schulman15} is summarized in the following theorem.
\begin{theorem}
\label{thm:trust_region}
Let $\epsilon = \max_{s,a}  |A^{\pi_\theta}(s,a)|$ 
and $\alpha = \max_{s} \frac{1}{2}\int |\pi_\theta (a|s) - \pi_{\tilde\theta}(a|s) |da$. Then the $J ( \pi_{\tilde\theta} )$ can be lower bounded as follows:
\begin{equation}
\label{mpi_lb}
J( \pi_{\tilde\theta} ) \ge L_{\pi_\theta}( \pi_{\tilde\theta} ) - \frac{4 \epsilon \gamma}{\left(1-\gamma\right)^2} \alpha^2.
\end{equation}
\end{theorem}
The work done in \cite{pmlr-v70-achiam17a} extends this results so that the $\max$ operator in $\alpha$ can be replaced with the expected value taken w.r.t $\rho^\pi$. Calculating the gradient $\frac{\partial L_{\pi_{\theta}}(\pi_{\tilde\theta})}{ \partial \pi_{\tilde\theta} }$ is straightforward as the occupancy measure $\rho_{\pi_\theta}$ in the definition of $L_{\pi_\theta}\left( \pi_{\tilde\theta} \right)$ does not depend on $\tilde\theta$.

In the practical setting $ L_{\pi_\theta}( \pi_{\tilde\theta}) $ is optimized by constraining the divergence between $\pi_\theta$ and $\pi_{\tilde\theta}$ \cite{pmlr-v37-schulman15, DBLP:journals/corr/SchulmanWDRK17}. The algorithm derived in \cite{pmlr-v37-schulman15} is closely similar to the natural gradient policy optimisation \cite{Kakade:2001:NPG:2980539.2980738}. Note that  $ L_{\pi_\theta}( \pi_{\tilde\theta}) $ is a biased approximation of $J(\pi_{\tilde\theta})$ but the bias can be controlled by restricting the distance between $\pi_{\tilde\theta}$ and $\pi_\theta$.

Optimizing $L_{\pi_\theta }( \pi_{\tilde\theta})$ requires knowing the values of $A^{\pi_\theta} (s,a)$. There are various ways of estimating the critic $\hat A^{\pi_\theta} (s,a)$, for instance see \cite{schulman2015high}. However, similarly to the case of \refthm{thm:pg_func_approximator}, using an arbitrary parametrization of the critic introduces bias to an estimate of  $\frac{\partial L_{\pi_{\theta}}( \pi_{\tilde\theta} )}{ \partial \pi_{\tilde\theta} }$ .

\section{Compatible features for surrogate policy optimization} 

In this section, we seek a parametric form $f_w (s,a)$ of an approximator of $Q^{\pi_\theta} (s,a)$ for which $\frac{\partial L_{\pi_{\theta}}( \pi_{\tilde\theta} )}{\partial\tilde\theta}$ remains unbiased. To this end, we follow the approach presented in \cite{Sutton:1999:PGM:3009657.3009806}. We derive the following theorem.
\begin{theorem}[Compatible features for Monotonic Policy Improvement]
\label{thm:compatible1}
 Assuming the following condition is satisfied:
\begin{equation}
\label{assumption:least_squares}
\int \rho_{\pi_\theta} (s) \int \pi_{\theta} (a|s) \big (Q^{\pi_\theta}(s,a) - f_w(s,a) \big ) \frac{\partial f_w  (s,a)}{\partial w}  da ds = 0
\end{equation}
and
\begin{equation}
\label{assumption:compatible}
f_w(s,a) = w^\top \frac{\pi_{\tilde\theta}(a|s) }{\pi_\theta(a|s) } \frac{\partial\log\pi_{\tilde\theta}(a|s) }{\partial\tilde\theta} + c_0. 
\end{equation}
Then,
\begin{equation}
\frac{\partial L_{\pi_\theta}( \pi_{\tilde\theta} )}{\partial\tilde\theta} = \int \rho_{\pi_\theta} (s)\int \frac{\partial\pi_{\tilde\theta}(a|s)}{\partial\tilde\theta} f_w (s,a ) da ds.
\end{equation}
\end{theorem}
\begin{proof}
Firstly, we note that the value function $V^{\pi_{\theta}}$ in the definition of $L_{\pi_\theta} ( \pi_{\tilde\theta})$ fulfils the role of the control variate, i.e. it does not influence the expectation of the gradient. To note this, we analyse the gradient $\frac{\partial L_{\pi_\theta}(\pi_{\tilde\theta})}{\partial\tilde\theta}$:
\begin{align}
\label{eq:surrogate_L}
\frac{\partial L_{\pi_\theta}( \pi_{\tilde\theta} )}{\partial\tilde\theta} 
& = \frac{\partial}{\partial\tilde\theta} \int \rho_{\pi_\theta} (s)\int \pi_{\tilde\theta} (a|s)  A^{\pi_\theta}(s,a ) da ds \notag\\
& = \int \rho_{\pi_\theta} (s)\int \frac{\partial \pi_{\tilde\theta} (a|s)}{\partial\tilde\theta} A^{\pi_\theta} (s,a) da ds   \notag\\
& = \int \rho_{\pi_\theta}(s)\int \frac{\partial \pi_{\tilde\theta} (a|s)}{\partial\tilde\theta}  \big ( Q^{\pi_\theta} (s,a) - V^{\pi_\theta}(s) \big )  da ds   \notag\\
& = \int \rho_{\pi_\theta}(s) \Bigg [ \int \frac{\partial \pi_{\tilde\theta} (a|s)}{\partial\tilde\theta} Q^{\pi_\theta} (s,a ) da - \int \frac{\partial \pi_{\tilde\theta} (a|s)}{\partial\tilde\theta}  V^{\pi_\theta} (s) da \Bigg ] ds   \notag\\
& = \int \rho_{\pi_\theta}(s)\int \frac{\partial \pi_{\tilde\theta}(a|s)}{\partial\tilde\theta} Q^{\pi_\theta} (s,a ) da ds 
,
\end{align}
where the second step is allowed since $\frac{\partial \pi_\theta(a|s)}{\partial \theta}$ exist and is continuous $\forall s \in S$, $\forall a \in A$; 
and the last step is due to the fact that $V^{\pi_\theta}$ does not depend on $\tilde\theta$, i.e.:  
$ \int \frac{\partial \pi_{\tilde\theta}(a|s)}{\partial\tilde\theta}  V^{\pi_\theta} (s) da  = 
\frac{\partial  V^{\pi_\theta}(s)}{\partial\tilde\theta} = 0$.

Next we define $\mathcal{L}_{\pi_\theta, w}( \pi_{\tilde\theta}) := \int \rho_{\pi_\theta} (s)\int \pi_{\tilde\theta} (a|s) f_w(s,a) da ds$. We seek a condition under which the following equality holds:
\begin{equation}
\label{eqn:objective}
\frac{\partial L_{\pi_\theta}}{\partial\tilde\theta} - \frac{\partial\mathcal{L}_{\pi_\theta, w}}{\partial\tilde\theta} = 0. 
\end{equation}
By subtracting $\frac{\partial\mathcal{L}_{\pi_\theta, w}}{\partial\tilde\theta}$ and the assumption given by \eqref{assumption:least_squares} from \eqref{eq:surrogate_L}, we obtain:
\begin{align}
& \frac{\partial L_{\pi_\theta}( \pi_{\tilde\theta})}{\partial\tilde\theta}- \frac{\partial\mathcal{L}_{\pi_\theta, w}( \pi_{\tilde\theta}) }{\partial\tilde\theta} 
    = \int \rho_{\pi_\theta} (s) \frac{\partial\pi_{\tilde\theta}(a|s)}{\partial\tilde\theta} \big(Q^{\pi_\theta}(s,a) -f_w(s,a)\big) da ds \nonumber \\
&\; = \int \rho_{\pi_\theta} (s) \int \pi_{\tilde\theta} (a|s) \frac{\partial\log\pi_{\tilde\theta}(a|s)}{\partial\tilde\theta} \big(Q^{\pi_\theta}(s,a) -f_w(s,a)\big) da ds \nonumber\\
&\quad - \int \rho_{\pi_\theta} (s) \int \pi_{\theta} (a|s) \big (Q^{\pi_\theta}(s,a) - f_w(s,a)\big ) \frac{\partial f_w(s, a) }{\partial w}da ds  \nonumber \\
&\; = \int \rho_{\pi_\theta} (s)\int \pi_{\theta} (a|s) \frac{\pi_{\tilde\theta}(a|s) }{\pi_\theta(a|s) } \frac{\partial\log\pi_{\tilde\theta}(a|s)}{\partial\tilde\theta} \big(Q^{\pi_\theta}(s,a) -f_w(s,a) \big) da ds \nonumber \\
&\quad - \int \rho_{\pi_\theta} (s) \int \pi_{\theta} (a|s) \big(Q^{\pi_\theta}(s,a) - f_w(s,a)\big) \frac{\partial f_w(s, a)}{\partial w}da ds  \nonumber \\
&\; = \int \rho_{\pi_\theta} (s) \int \pi_{\theta} (a|s) \big(Q^{\pi_\theta}(s,a) - f_w(s,a)\big) \Bigg [\frac{\pi_{\tilde\theta}(a|s ) }{\pi_\theta(a|s) } \frac{\partial\log\pi_{\tilde\theta}(a|s)}{\partial\tilde\theta} - \frac{\partial f_w(s, a)}{\partial w}  \Bigg ] dads .
\end{align}
where we have used the log trick, 
$ \pi_{\tilde\theta}(a|s) \frac{\partial \log \pi_{\tilde\theta}(a|s)} {\partial \tilde\theta} =  \frac{\partial \pi_{\tilde\theta}(a|s)}  {\partial \tilde\theta}$, in the second line. Hence \refeqn{eqn:objective} can by satisfied by requiring:
\begin{equation}
\frac{\pi_{\tilde\theta}}{\pi_\theta} \frac{\partial\log\pi_{\tilde\theta}}{\partial\tilde\theta}- \frac{\partial f_w}{\partial w} = 0 .
\end{equation}
Integrating this last equation w.r.t $w$ yields:
$
f_w(s,a) = w^\top \frac{\pi_{\tilde\theta}(a|s) }{\pi_\theta(a|s) } \frac{\partial\log\pi_{\tilde\theta}(a|s)}{\partial\tilde\theta} + c_0
$,
which completes the proof.
\end{proof}

Again setting $c_0 = V^{\pi_{\theta}}(s)$ ensures that:
\begin{align}
    \mathbb{E}_{a \sim \pi_\theta(\cdot|s) } f_w(s,a) 
&= 
    w^\top \mathbb{E}_{a \sim \pi_\theta(\cdot|s) } \frac{\pi_{\tilde\theta}(a|s) }{\pi_\theta(a|s) } \frac{\partial\log\pi_{\tilde\theta}(a|s)}{\partial\tilde\theta} +  V^{\pi_{\theta}}(s) 
\notag\\
&=  
    w^\top \mathbb{E}_{a \sim \pi_{\tilde\theta}(\cdot|s) } \frac{\partial\log\pi_{\tilde\theta}(a|s)}{\partial\tilde\theta} +  V^{\pi_{\theta}}(s) 
\notag = w^\top \frac{\partial  \int \pi_{\tilde\theta}(a|s) da  }{\partial \tilde \theta}+ V^{\pi_{\theta}}(s) 
\notag\\
&= 
    V^{\pi_{\theta}}(s)
,\quad 
    \forall s \in S, \forall a \in A.
\end{align}
In reference \cite{Sutton:1999:PGM:3009657.3009806} the authors conjecture that the the compatible features $\frac{\partial \log \pi_\theta (a|s)}{\partial \theta}$ might be the only choice of features that lead to an unbiased policy gradient $\frac{\partial J (\pi_\theta)}{\partial\theta}$. In the case of $L_{\pi_\theta}(\pi_{\tilde \theta} )$ we can provide another choice of features leading to an unbiased gradient $\frac{\partial L_{\pi_\theta}(\pi_{\tilde\theta})}{\partial\tilde\theta}$. The features derived in \refthm{thm:compatible1} depend on importance sampling weights $\frac{\pi_{\tilde\theta}(a|s)}{\pi_{\tilde\theta}(a|s)}$. To remove importance sampling weights $\frac{\pi_{\tilde\theta}(a|s)}{\pi_{\tilde\theta}(a|s)}$ from derived compatible features we modify the condition in \refeqn{assumption:least_squares} by replacing action sampling distribution $\pi_\theta (\cdot |s)$ with $\pi_{\tilde\theta} (\cdot |s)$.
\begin{theorem}
\label{thm:compatible2}
Assuming that the following condition is satisfied:
\begin{equation}
\label{assumption:least_squares2}
\int \rho_{\pi_\theta} (s) \int \pi_{\tilde\theta} (a|s) \big(Q^{\pi_\theta}(s,a) - f_w(s,a) \big) \frac{\partial f_w (s, a)}{\partial w}  da ds = 0
\end{equation}
and
\begin{equation}
\label{features:compatible2}
f_w(s,a) = w^\top  \frac{\partial \log \pi_{\tilde\theta}(a|s)}{\partial \tilde \theta} + c_0.
\end{equation}
Then,
\begin{equation}
\frac{\partial L_{\pi_\theta}( \pi_{\tilde\theta} )}{\partial\tilde\theta} = \int \rho_{\pi_\theta} (s)\int \frac{\partial\pi_{\tilde\theta}(a|s)}{\partial\tilde\theta}  f_w (s,a) da ds.
\end{equation}
\end{theorem}

\begin{proof}
We derive the result by following a similar line of thought as in proof of \refthm{thm:compatible1}. Again we use $\mathcal{L}_{\pi_\theta, w}\left( \pi_{\tilde\theta} \right) = \int \rho_{\pi_\theta} (s)\int \pi_{\tilde\theta} (a|s ) f_w(s,a) da ds$. We subtract $\frac{\partial\mathcal{L}_{\pi_\theta, w}( \pi_{\tilde\theta})}{\partial\tilde\theta}$ and the assumption given by \eqref{assumption:least_squares2} from $\frac{\partial L_{\pi_\theta}( \pi_{\tilde\theta})}{\partial\tilde\theta}$, which yields:
\begin{align}
\label{eqn:compatible2_proof}
    & \frac{\partial L_{\pi_\theta}( \pi_{\tilde\theta})}{\partial\tilde\theta} - \frac{\partial\mathcal{L}_{\pi_\theta, w}( \pi_{\tilde\theta})}{\partial\tilde\theta} = \notag \\
    &\; = \int \rho_{\pi_\theta} (s) \int \pi_{\tilde\theta} (a|s) \frac{\partial  \log \pi_{\tilde\theta} (a|s)}{\partial \tilde \theta} \big(Q^{\pi_\theta}(s,a) - f_w(s,a) \big) \notag  \\
    &\quad - \int \rho_{\pi_\theta} (s) \int \pi_{\tilde\theta} (a|s) \big(Q^{\pi_\theta}(s,a) - f_w(s,a) \big) \frac{\partial f_w (s, a)}{\partial w}  da ds  \notag \\
    &\; = \int \rho_{\pi_\theta} (s) \int \pi_{\tilde\theta} (a|s) \big(Q^{\pi_\theta}(s,a) - f_w(s,a) \big) \Bigg [ \frac{\partial  \log \pi_{\tilde\theta} (a|s)}{\partial \tilde \theta}  -  \frac{\partial f_w (s,a)}{\partial w}   \Bigg ] da ds
.
\end{align}
The integral in \refeqn{eqn:compatible2_proof} can be set to zero by setting $f_w$ to solve the differential equation:
$
    \frac{\partial  \log \pi_{\tilde\theta} }{\partial \tilde \theta}  -  \frac{\partial f_w}{\partial w} = 0,
$,
which results in desired compatible features $f_w(s,a) = w^\top  \frac{\partial \log \pi_{\tilde\theta}(a|s)}{\partial \tilde \theta} + c_0$.
\end{proof}

Note the features derived in \refthm{thm:compatible2} are analogous to ones derived in \refthm{thm:pg_func_approximator} with the difference that the occupancy measure is taken w.r.t the data gathering policy $\pi_{\theta}$ in the case of \refthm{thm:compatible2}.

Assumption \eqref{assumption:least_squares2} in \refthm{thm:compatible2} has a natural interpretation. Intuitively, the critic values $f_w(s,a)$ should be more accurate for actions with high probability under policy $\pi_{\tilde\theta}$. More formally, the condition in \eqref{assumption:least_squares2} can be satisfied by finding a parameter $w^*$ that minimises the weighted quadratic error:
$w^* = \arg\min \int \rho_{\pi_\theta} (s) \int \pi_{\theta} (a|s) \frac{\pi_{\tilde\theta}(a|s)}{\pi_{\tilde\theta}(a|s)} (Q^{\pi_\theta}(s,a) - f_w(s,a))^2 da ds$, 
where weights $\frac{\pi_{\tilde\theta}(a|s)}{\pi_{\tilde\theta}(a|s)}$ ensure improved accuracy for likely actions $a$ given $s$ under target policy $\pi_{\tilde\theta} (\cdot |s)$.  
Similarly, when target policy $\pi_{\tilde\theta} (\cdot |s)$ assigns low probability to an action $a$ the error in the critic $f_w(s,a)$ estimation is not relevant as it will not influence the gradient $\frac{\partial L_{\pi_\theta}(\pi_{\tilde\theta})}{\partial\tilde\theta}$.

Given a sequence of state action pairs $\{(s_t, a_t) \}_{t=1}^T$ gathered by following policy $\pi_\theta$,
the integrals in condition \eqref{assumption:least_squares2} can be also approximated with samples, which leads to a weighted regression problem emerging from \refthm{thm:compatible2}:
\begin{equation}
\label{eq:weighted_regression}
    w^* = \text{argmin}_w \frac{1}{T} \sum_{s,a}\frac{\pi_{\tilde\theta} (a|s)}{\pi_\theta(a|s)} \big(Q^{\pi_\theta}(s,a) - f_w(s,a)\big)^2.
\end{equation}

\section{Experiments}
In this section we use the NChain \cite{Strens:2000:BFR:645529.658114} environment to compare the gradient $\frac{\partial L_{\pi_\theta}( \pi_{\tilde\theta} )}{\partial\tilde\theta}$ calculated in two different ways: 
i) using a standard linear critic of the form $f_w(s,a) = w_2 s + w_1 a + w_0$, learned with the standard least squares regression approach given by \eqref{eq:standard_regression}; 
and ii) using a linear critic with compatible features given by \eqref{features:compatible2}, learned by solving the weighted the least squared regression problem \eqref{eq:weighted_regression}.
We use the following parametric policy function $\pi_\theta(a|s) = \sigma( \theta_1 s + \theta_2 a)$. 
Hence, the proposed compatible features are given by $\frac{\partial \log \pi_{\tilde\theta}(a|s)}{\partial \tilde \theta} = \big [\big(1-\sigma( \theta_1 s + \theta_2 a) \big ) s,  \big (1-\sigma(\theta_1 s + \theta_2 a) \big ) a \big ]^\top$, 
since $\frac{d\sigma(x)}{d x} = \sigma(x)\big(1-\sigma(x)\big )$ and  $\frac{d \log\sigma(f(x))}{d x} = \big(1-\sigma(f(x))\big ) \frac{df(x)}{dx}$. 
We provide true state action value function $Q^{\pi_\theta}(s,a)$ as targets to learn critics. We estimate the gradient $\frac{\partial L_{\pi_\theta}( \pi_{\tilde\theta} )}{\partial\tilde\theta}$ using rollouts from policy $\pi_\theta$. We set policy parameters to $\theta_1 = 0.2$; $\theta_2=0.5$ and $\tilde\theta_1 = 0.3$; $\tilde\theta_2 = 0.6$. 
We use the expression for $\frac{\partial L_{\pi_\theta}( \pi_{\tilde\theta} )}{\partial\tilde\theta}$ given by \refeqn{eq:surrogate_L}.

We compare the obtained gradients to ground truth calculated using the true state action value function $Q^{\pi_\theta}(s,a)$. 
In both cases, we provide an increasing number of rollouts from policy $\pi_\theta$ to both approaches and analyse the errors in estimation. We report bias, variance and RMSE of estimated gradients in Figure 1. 
\begin{figure}[h!]
\centering
 \includegraphics[width=85mm, height=50mm]{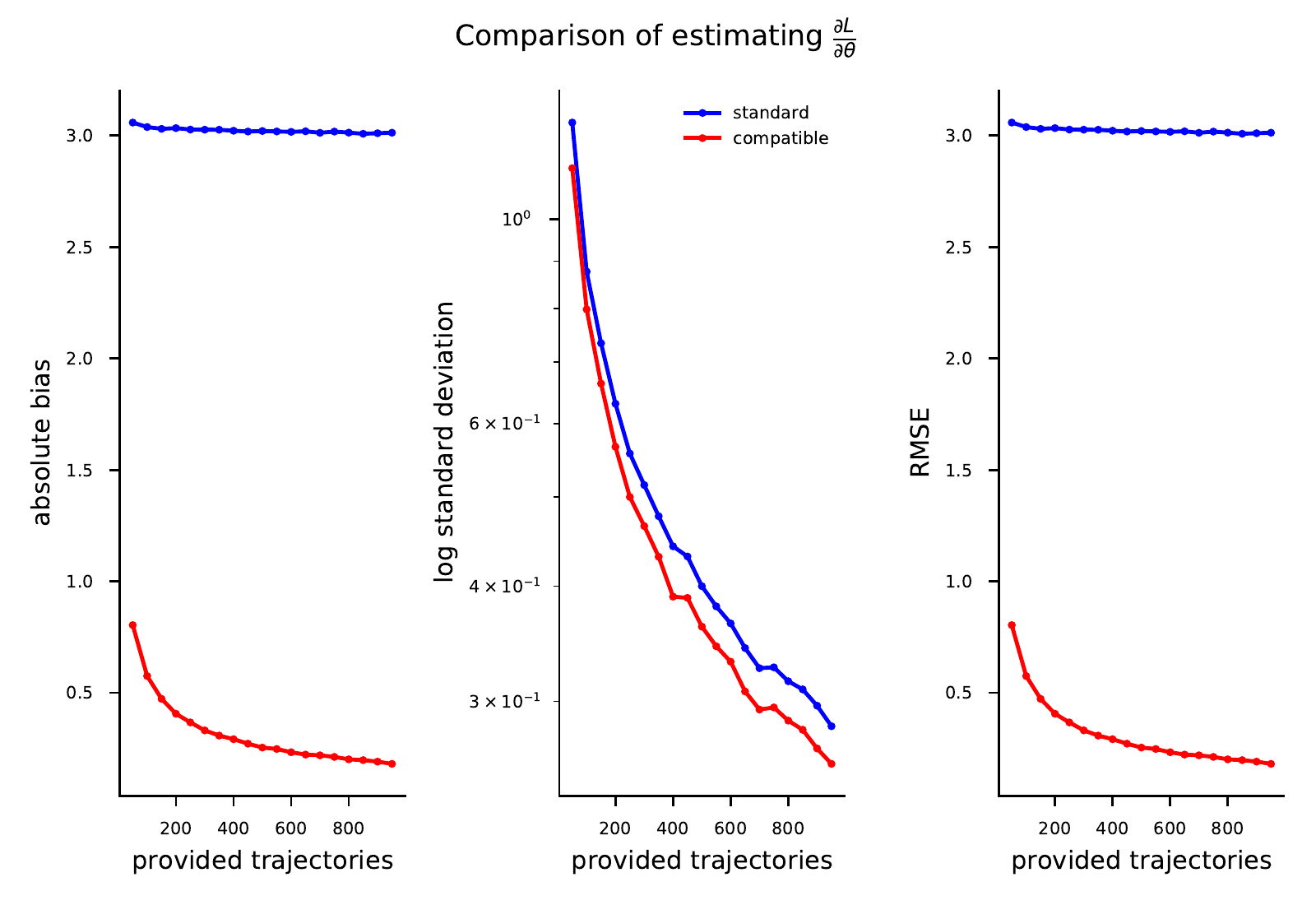}
 \caption{Comparison of gradient $\frac{\partial L_{\pi_\theta}( \pi_{\tilde\theta} )}{\partial\tilde\theta}$ derived with compatible parametrization and by using standard critic. Employing a critic with derived compatible features allows to provide unbiased gradient of $\frac{\partial L_{\pi_\theta}( \pi_{\tilde\theta} )}{\partial\tilde\theta}$. The results are averaged over $250$ trials.}
\end{figure}

As expected, the bias in the estimation of $\frac{\partial L_{\pi_\theta}( \pi_{\tilde\theta} )}{\partial\tilde\theta}$ introduced by the standard critic cannot be removed by increasing the number of provided trajectories, as it is caused by employing a noncompatible function approximator. The derived compatible features are unbiased by construction, hence provide substantially improved quality of estimation of $\frac{\partial L_{\pi_\theta}( \pi_{\tilde\theta} )}{\partial\tilde\theta}$.
Interestingly, using compatible features also provides slightly lower variance than using the standard critic.

\section{Conclusions}
We have analysed the use of parametric critics to estimate the policy optimization surrogate objective in a systematical way. As a consequence,  we can provide conditions under which the approximation does not introduce bias to the policy updates. This result holds for a general class of policies, including policies parametrized by deep neural networks. 
We have shown that for the investigated surrogate objective there exists two different choices of compatible features. We empirically demonstrated that the  compatible features allow to estimate the gradient of surrogate objective more accurately.

\bibliographystyle{unsrt}
\bibliography{neurips_2019}

\end{document}